\documentclass[letterpaper,11pt]{article}
\newcommand{\beas}{\begin{eqnarray*}}
\newcommand{\eeas}{\end{eqnarray*}}

{\endtrivlist}

\usepackage[ruled,vlined]{algorithm2e}
\usepackage{graphicx}
\usepackage{subfigure}
\usepackage{pifont}
\usepackage{float}
\usepackage[ruled,vlined]{algorithm2e}
\usepackage[all]{xy}
\usepackage{url}
\usepackage{fullpage}
\usepackage{setspace}
\usepackage{amsthm, amsmath, amssymb}
\usepackage{stmaryrd}
\usepackage[lmargin = 0.75in, rmargin = 0.75in, tmargin = 0.75in, bmargin = 0.7in]{geometry}

\newtheorem{thm}{Theorem}

\newtheorem{cor}{Corollary}
\newtheorem{lemma}{Lemma}
\newtheorem{defn}{Definition}

\theoremstyle{plain}

\newcommand{\cP}{\Pi}
\newcommand{\cI}{\mathcal{I}}
\newcommand{\cL}{\mathcal{L}}
\newcommand{\cT}{\mathcal{T}}
\newcommand{\cD}{\mathcal{D}}
\newcommand{\B}{\mathbf{B}}
\newcommand{\expect}{\mathbb{E}}
\newcommand{\prob}{\mathrm{Pr}}
\newcommand{\vd}{\mathbf{d}}
\newcommand{\vy}{\mathbf{y}}

\newcommand{\Lt}{\widetilde{L}}
\newcommand{\Ht}{\widetilde{H}}

\begin{document}

\title{\textbf{Query Learning with Exponential Query Costs}}

\date{}
\author{\small Gowtham Bellala$^{1}$, Suresh K. Bhavnani$^{2}$, Clayton D. Scott$^{1,2}$ \\ \small $^{1}$Department of Electrical Engineering and Computer Science,\\ \small $^2$Center for Computational Medicine and Bioinformatics, University of Michigan, Ann Arbor, MI 48109 \\ \small E-mail: \{ gowtham, bhavnani, clayscot \}@umich.edu}

\maketitle

\begin{abstract}
 In query learning, the goal is to identify an unknown object while minimizing the number of ``yes'' or ``no'' questions (queries) posed about that object. A well-studied algorithm for query learning is known as generalized binary search (GBS). We show that GBS is a greedy algorithm to optimize the expected number of queries needed to identify the unknown object. We also generalize GBS in two ways. First, we consider the case where the cost of querying grows exponentially in the number of queries and the goal is to minimize the expected exponential cost. Then, we consider the case where the objects are partitioned into groups, and the objective is to identify only the group to which the object belongs. We derive algorithms to address these issues in a common, information-theoretic framework. In particular, we present an exact formula for the objective function in each case involving Shannon or R\'{e}nyi entropy, and develop a greedy algorithm for minimizing it. Our algorithms are demonstrated on two applications of query learning, active learning and emergency response.
\end{abstract}


\section{Introduction}

In query learning there is an unknown object $\theta$ belonging to a set $\Theta = \{\theta_1,\cdots,\theta_M\}$ of $M$ different objects and a set $Q = \{q_1,\cdots,q_N\}$ of $N$ distinct subsets of $\Theta$ known as queries. Additionally, the vector $\cP = (\pi_1,\cdots,\pi_M)$ denotes the \emph{a priori} probability distribution over $\Theta$. The goal is to determine the unknown object $\theta \in \Theta$ through as few queries from $Q$ as possible, where a query $q \in Q$ returns a value $1$ if $\theta \in q$, and $0$ otherwise. A query learning algorithm thus corresponds to a binary decision tree, where the internal nodes are queries, and the leaf nodes are objects. The above problem is motivated by several real-world applications including fault testing \cite{koren,unluyurt}, machine diagnostics \cite{shiozaki}, disease diagnosis \cite{loveland,pattipati}, computer vision \cite{geman} and active learning \cite{dasgupta,nowak}. Algorithms and performance guarantees have been extensively developed in the literature, as described in Section \ref{sec:related work} below. We also note that the above problem is known more specifically as query learning with membership queries. See \cite{angluin} for an overview of query learning in general.

As a motivating example, consider the problem of toxic chemical identification, where a first responder may question victims of chemical exposure regarding the symptoms they experience. Chemicals that are inconsistent with the reported symptoms may then be eliminated. Given the importance of this problem, several organizations have developed extensive evidence-based databases (e.g., Wireless Information System for Emergency Responders (WISER) \cite{szczur}) that record toxic chemicals and the acute symptoms which they are known to cause. Unfortunately, many symptoms tend to be nonspecific (e.g., nausea can be caused by many different chemicals), and it is therefore critical for the first responder to pose these questions in a sequence that leads to chemical identification in as few questions as possible.

A well studied algorithm for query learning is known as the \emph{splitting algorithm} \cite{loveland} or \emph{generalized binary search} (GBS) \cite{dasgupta,nowak}. This is a greedy algorithm which selects a query that most evenly divides the probability mass of the remaining objects \cite{loveland,goodman,dasgupta,nowak}. In this paper, we consider two important limitations of GBS and propose natural extensions inspired from an information theoretic perspective.

First, we note that GBS is tailored to minimize the average number of queries needed to identify $\theta$, thereby implicitly assuming that the incremental cost for each additional query is constant. However, in certain applications, the cost of additional queries grows. For example, in time critical applications such as toxic chemical identification, each additional symptom queried impacts a first responder's ability to save lives. If some chemicals are less prevalent, GBS may require an unacceptably large number of queries to identify them. This problem is compounded when the prior probabilities $\pi_i$ are inaccurately specified. To address these issues, we consider an objective function where the cost of querying grows exponentially in the number of queries. This objective function has been used earlier in the context of source coding for the design of prefix-free codes (discussed in Section \ref{sec:related work}). We propose an extension of GBS that greedily optimizes this exponential cost function. The proposed algorithm is also intrinsically more robust to misspecification of the prior probabilities.

Second, we consider the case where the object set $\Theta$ is partitioned into groups of objects and it is only necessary to identify the group to which the object belongs. This problem is once again motivated by toxic chemical identification where the appropriate response to a toxic chemical may only depend on the class of chemicals to which it belongs (pesticide, corrosive acid, etc.). As we explain below, a query learning algorithm such as GBS that is designed to rapidly identify individual objects is not necessarily efficient for group identification. Thus, we propose a natural extension of GBS for rapid group identification. Once again, we consider an objective where the cost of querying grows exponentially in the number of queries.


\subsection{Background and related work}
\label{sec:related work}

The goal of a standard query learning problem is to construct an optimal binary decision tree, where each internal node in the tree is associated with a query from the set $Q$, and each leaf node corresponds to an object from $\Theta$. Optimality is often with respect to the expected number of queries needed to identify $\theta$, that is, the expected depth of the leaf node corresponding to the unknown object $\theta$. In the special case when the query set $Q$ is \emph{complete}\footnote{A query set $Q$ is said to be \emph{complete} if for any $S \subseteq \Theta$ there exists a query $q \in Q$ such that either $q = S$ or $\Theta \setminus q = S$}, the problem of constructing an optimal binary decision tree is equivalent to constructing optimal variable-length binary prefix-free codes with minimum expected length. This problem has been widely studied in information theory with both Shannon \cite{shannon} and Fano \cite{fano} independently proposing a top-down greedy strategy to construct suboptimal binary prefix codes, popularly known as Shannon-Fano codes. Huffman \cite{huffman} derived a simple bottom-up algorithm to construct optimal binary prefix codes. A well known lower bound on the expected length of the optimal binary prefix codes is given by the Shannon entropy of $\cP$ \cite{cover}.

The problem of query learning when the query set $Q$ is not \emph{complete} has also been studied extensively in the literature with Garey \cite{garey1,garey2} proposing an optimal dynamic programming based algorithm. This algorithm runs in exponential time in the worst case. Later, Hyafil and Rivest \cite{rivest} showed that determining an optimal binary decision tree for this problem is NP-complete. Thereafter, various greedy algorithms \cite{loveland,kosaraju,roy} have been proposed to obtain a suboptimal binary decision tree. A widely studied solution is the \emph{splitting algorithm} \cite{loveland} or \emph{generalized binary search} (GBS) \cite{dasgupta,nowak}. Various bounds on the performance of this greedy algorithm have been established in \cite{loveland,dasgupta,nowak}. We show below in Corollary \ref{cor:OIsc1} that GBS greedily minimizes the average number of queries, and thus weights each additional query by a constant.

Here, we consider an alternate objective function where the cost grows exponentially in the number of queries. Specifically, the objective function is given by $L_{\lambda}(\cP,\vd) := \log_{\lambda} \left ( \sum_{i=1}^M \pi_i \lambda^{d_i} \right )$, where $\lambda \geq 1$ and $\vd = (d_1,\cdots,d_M)$, $d_i$ corresponding to the number of queries required to identify object $\theta_i$ using a given tree.  This cost function was proposed by Campbell \cite{campbell1} in the context of source coding for the design of binary prefix-free codes. It has also been used recently for the design of alphabetic codes \cite{baer1} and random search trees \cite{schulz}. 

Campbell \cite{campbell2} defines a generalized entropy function in terms of a coding problem and shows that the $\alpha$-R\'{e}nyi entropy, given by $H_{\alpha}(\cP) = \frac{1}{1-\alpha} \log_2 \left ( \sum_{i=1}^M \pi_i^{\alpha} \right )$, can be characterized as 
\begin{align}
\label{eq:renyi entropy definition}
H_{\alpha} (\cP) & = \underset{\vd \in S}{\operatorname{\inf}} L_{\lambda}(\cP, \vd)
\end{align}
where $\alpha = \frac{1}{1 + \log_2\lambda}$ and $S$ is the set of all real distributions of $\vd$ for which the Kraft's inequality $\sum_{i=1}^M 2^{-d_i} \leq 1$, is satisfied. For clarity, here we show the dependence of $L_{\lambda}(\cP)$ on $\vd$, although later this dependence will not be made explicit. Note that the numbers $d_i$ are not restricted to integer values in (\ref{eq:renyi entropy definition}), hence the R\'{e}nyi entropy merely provides a lower bound on the exponential cost function of any binary decision tree. In the special case when the query set $Q$ is \emph{complete}, it has been shown that an optimal binary decision tree (i.e., optimal binary prefix-free codes) that minimizes $L_{\lambda}(\cP)$ can be obtained by a modified version of the Huffman algorithm \cite{hu,parker,humblet,schulz}. However, when the query set $Q$ is not \emph{complete}, there does not exist an algorithm to the best of our knowledge that constructs a good suboptimal decision tree.



\subsection{Notation}
\label{sec:notation}
We denote a query learning problem by a pair $(\B,\cP)$ where $\B$ is a known $M \times N$ binary matrix with $b_{ij}$ equal to $1$ if $\theta_i \in q_j$, and $0$ otherwise. A decision tree $T$ constructed on $(\B,\cP)$ has a query from the set $Q$ at each of its internal nodes, with the leaf nodes terminating in the objects from $\Theta$. At each internal node in the tree, the objects that have reached that node are divided into two subsets, depending on whether they respond $0$ or $1$ to the query, respectively. For a decision tree with $L$ leaves, the leaf nodes are indexed by the set $\cL = \{1,\cdots,L\}$ and the internal nodes are indexed by the set $\cI = \{L+1,\cdots,2L-1\}$. At any internal node $a \in \cI$, let $l(a), r(a)$ denote the ``left'' and ``right'' child nodes, and let $\Theta_a \subseteq \Theta$ denote the set of objects that reach node `$a$'. Thus, the sets $\Theta_{l(a)} \subseteq \Theta_a, \Theta_{r(a)} \subseteq \Theta_a$ correspond to the objects in $\Theta_a$ that respond $0$ and $1$ to the query at node `$a$', respectively. We denote by $\pi_{\Theta_a} := \sum_{\{i: \theta_i \in \Theta_a \}}\pi_i$, the probability mass of the objects reaching node `$a$' in the tree. Also, at any node `$a$', the set $Q_a \subseteq Q$ corresponds to the set of queries that have been performed along the path from the root node up to node `$a$'.  

We denote the $\alpha$-R\'{e}nyi entropy of a vector $\cP = (\pi_1,\cdots,\pi_M)$ by $H_{\alpha}(\Pi) := \frac{1}{1-\alpha}\log_2 \left ( \sum_{i=1}^M \pi_i^{\alpha} \right )$ and its Shannon entropy by $H(\cP) := - \sum_i \pi_i \log_2 \pi_i$, where we use the limit, $\displaystyle \lim_{\pi \rightarrow 0} \pi \log_2 \pi = 0$ to define the limiting cases as $\pi_i \to 0$ for any $i$. Using L'H\^{o}pital's rule, it can be seen that $\underset{\alpha \to 1}{\operatorname{\lim}} H_{\alpha}(\cP) = H(\cP)$. Also, we denote the Shannon entropy of a proportion $\pi \in [0,1]$ by $H(\pi) := -\pi \log_2 \pi - (1-\pi) \log_2 (1-\pi)$. 


\section{Object Identification}
\label{sec:OI}

We begin with the basic query learning problem where the goal is to identify the unknown object $\theta \in \Theta$ in as few queries from $Q$ as possible. We propose a family of greedy algorithms to minimize the exponential cost function $L_{\lambda}(\cP)$ where $\lambda \geq 1$. These algorithms are based on Theorem \ref{thm:OI}, which provides an explicit formula for the gap in Campbell's lower bound. We also note that $L_{\lambda}(\cP)$ reduces to the average depth and the worst case depth in the limiting cases when $\lambda$ tends to one and infinity, respectively. In particular,
\begin{align*}
 L_1(\cP) & := \underset{\lambda \rightarrow 1}{\operatorname{\lim}} L_{\lambda}(\Pi) = \sum_{i=1}^M \pi_id_i \\
L_{\infty}(\cP) &:= \underset{\lambda \rightarrow \infty}{\operatorname{\lim}} L_{\lambda}(\Pi) = \underset{i \in \{1,\cdots,M \}}{\operatorname{\max}} d_i 
\end{align*} 
where $d_i$ denotes the number of queries required to identify object $\theta_i$ in a given tree. In these limiting cases, the entropy lower bound on the cost function reduces to the Shannon entropy $H(\cP)$ and $\log_2 M$, respectively. 

Given a query learning problem $(\B,\cP)$, let $\cT(\B,\cP)$ denote the set of decision trees that can uniquely identify all the objects in the set $\Theta$. 
\begin{thm}
\label{thm:OI}
For any $\lambda \geq 1$, the average exponential depth of the leaf nodes $L_{\lambda}(\cP)$ in a tree $T \in \cT(\B,\cP)$ is given by
\begin{align}
\label{eq:OI}
L_{\lambda}(\cP) = \log_{\lambda} \left ( \lambda^{H_{\alpha}(\Pi)} + \sum_{a \in \cI} \pi_{\Theta_{a}} \left [ (\lambda - 1)\lambda^{d_a} - \cD_{\alpha}(\Theta_a) + \frac{\pi_{\Theta_{l(a)}}}{\pi_{\Theta_a}} \cD_{\alpha}(\Theta_{l(a)}) + \frac{\pi_{\Theta_{r(a)}}}{\pi_{\Theta_a}} \cD_{\alpha}(\Theta_{r(a)})  \right ] \right )
\end{align}
where $d_a$ denotes the depth of any internal node `$a$' in the tree, $\Theta_a$ denotes the set of objects that reach node `$a$', $\pi_{\Theta_a} = \underset{\{i:\theta_i \in \Theta_a\}}{\operatorname{\sum}} \pi_i$, $\alpha = \frac{1}{1 + \log_2 \lambda}$ and $\cD_{\alpha}(\Theta_a) := \left [ \sum_{\{i:\theta_i \in \Theta_a\}} \left ( \frac{\pi_i}{\pi_{\Theta_a}} \right )^{\alpha} \right ]^{1/\alpha}$.
\end{thm}
\begin{proof}
Special case of Theorem \ref{thm:GI} below.
\end{proof}

Theorem \ref{thm:OI} provides an explicit formula for the gap in the Campbell's lower bound, namely, the term in summation over internal nodes $\cI$ in (\ref{eq:OI}). Using this theorem, the problem of finding a decision tree with minimum $L_{\lambda}(\Pi)$ can be formulated as the following optimization problem:
\begin{align*}
\underset{T \in \cT(\B,\cP)}{\operatorname{\min}} \log_{\lambda} \left ( \lambda^{H_{\alpha}(\Pi)} + \sum_{a \in \cI} \pi_{\Theta_{a}} \left [ (\lambda - 1)\lambda^{d_a} - \cD_{\alpha}(\Theta_a) + \frac{\pi_{\Theta_{l(a)}}}{\pi_{\Theta_a}} \cD_{\alpha}(\Theta_{l(a)}) + \frac{\pi_{\Theta_{r(a)}}}{\pi_{\Theta_a}} \cD_{\alpha}(\Theta_{r(a)}) \right ] \right ).
\end{align*}
Since $\log_{\lambda}$ is a monotonic increasing function and $\Pi$ is fixed for a given problem, the above optimization problem can be reduced to
\begin{eqnarray}
\label{eq:optimization OI}
& \underset{T \in \cT(\B,\cP)}{\operatorname{\min}}  \sum_{a \in \cI} \pi_{\Theta_{a}} \left [ (\lambda - 1)\lambda^{d_a} - \cD_{\alpha}(\Theta_a) + \frac{\pi_{\Theta_{l(a)}}}{\pi_{\Theta_a}} \cD_{\alpha}(\Theta_{l(a)}) + \frac{\pi_{\Theta_{r(a)}}}{\pi_{\Theta_a}} \cD_{\alpha}(\Theta_{r(a)}) \right ]. &
\end{eqnarray}
As we show in Section \ref{sec:OIsc1}, this optimization problem is a generalized version of an optimization problem that is NP-complete. Hence, we propose a suboptimal approach to solve this optimization problem where we minimize the objective function locally instead of globally. We take a top-down approach and minimize the objective function by minimizing the term $\pi_{\Theta_{a}} \left [ (\lambda - 1)\lambda^{d_a} - \cD_{\alpha}(\Theta_a) + \frac{\pi_{\Theta_{l(a)}}}{\pi_{\Theta_a}} \cD_{\alpha}(\Theta_{l(a)}) + \frac{\pi_{\Theta_{r(a)}}}{\pi_{\Theta_a}} \cD_{\alpha}(\Theta_{r(a)}) \right ]$ at each internal node, starting from the root node. Note that the terms that depend on the query chosen at node `$a$' are $\pi_{\Theta_{l(a)}},\pi_{\Theta_{r(a)}},\cD_{\alpha}(\Theta_{l(a)})$ and $\cD_{\alpha}(\Theta_{r(a)})$. Hence, the objective function to be minimized at each internal node reduces to $C_a:= \frac{\pi_{\Theta_{l(a)}}}{\pi_{\Theta_a}} \cD_{\alpha}(\Theta_{l(a)}) + \frac{\pi_{\Theta_{r(a)}}}{\pi_{\Theta_a}} \cD_{\alpha}(\Theta_{r(a)})$. The algorithm, which we refer to as $\lambda$-GBS, can be summarized as shown in Algorithm \ref{algo_OI}.

\restylealgo{boxed}
\begin{algorithm}
\dontprintsemicolon
\caption{Greedy decision tree algorithm for minimizing average exponential depth \label{algo_OI}}
\textbf{\underline{$\lambda$-GBS}} \;
\BlankLine
\textbf{Initialization :}  \emph{Let the leaf set consist of the root node}, $Q_{root} = \emptyset$ \;
\SetLine
\While{some leaf node `$a$' has $|\Theta_a| > 1$}{
\For{each query $q \in Q \setminus Q_a$}{
Find $\Theta_{l(a)}$ and $\Theta_{r(a)}$ produced by making a split with query $q$\;
Compute the cost $C_a(q)$ of  making a split with query $q$ \;
}
Choose a query with the least cost $C_a$ at node `$a$' \;
Form child nodes $l(a),r(a)$\;
}
\end{algorithm} 

In the following two sections, we show that in the limiting case when $\lambda$ tends to one, where the average exponential depth reduces to the average linear depth, $\lambda$-GBS reduces to GBS, and in the case when $\lambda$ tends to infinity, $\lambda$-GBS reduces to GBS with the uniform prior $\pi_i = 1/M$.


\subsection{Average case}
\label{sec:OIsc1}
We now present with an exact formula for the average number of queries $L_1(\cP)$ required to identify an unknown object $\theta$ using a given tree, and show that GBS is a greedy algorithm to minimize this expression. First, we define a parameter called the \emph{reduction factor} on the binary matrix/tree combination that provides a useful quantification of the cost function $L_1(\cP)$.
\begin{defn}[\textbf{Reduction factor}] Let $T$ be a decision tree constructed on the query learning problem $(\B,\cP)$. The \emph{reduction factor} at any internal node `$a$' in the tree is defined by $\rho_a = \max \{\pi_{\Theta_{l(a)}},\pi_{\Theta_{r(a)}}\}/\pi_{\Theta_a}$.
\end{defn}
Note that $0.5 \leq \rho_a \leq 1$.

\begin{cor}
\label{cor:OIsc1}
The expected number of queries required to identify an unknown object using a tree $T \in \cT(\B,\cP)$ is given by
\begin{equation}
\label{eq:OIsc1}
L_1(\cP) = H(\cP) + \sum_{a \in \cI} \pi_{\Theta_{a}}[1 - H(\rho_a)] 
\end{equation}
where $H(\cdot)$ denotes the Shannon entropy.
\end{cor}
\begin{proof}
The result follows from Theorem \ref{thm:OI} by taking the limit as $\lambda$ tends to $1$ and applying L'H\^{o}pital's rule on both sides of (\ref{eq:OI}).
\end{proof}

This corollary re-iterates an earlier observation that the expected number of queries required to identify an unknown object using a tree $T$ is bounded below by the Shannon entropy $H(\cP)$. Besides, it presents the exact formula for the gap in this lower bound. It also follows from the above result that a tree attains this minimum value (i.e., $L_1(\cP) = H(\cP)$) iff the reduction factors are equal to $0.5$ at each internal node in the tree.

Using this result, the problem of finding a decision tree with minimum $L_1(\cP)$ can be formulated as the following optimization problem:
\begin{eqnarray}
\label{eq:optimization OIsc1}
& \underset{T \in \cT(\B,\cP)}{\operatorname{\min}} H(\cP ) + \sum_{a \in \cI} \pi_{\Theta_a}[1 - H(\rho_a )]. &
\end{eqnarray}
Since $\cP$ is fixed, this optimization problem reduces to minimizing $\sum_{a \in \cI} \pi_{\Theta_a}[1 - H(\rho_a)]$ over the set of trees $\cT(\B,\cP)$. Note that this optimization problem is a special case of the optimization problem in (\ref{eq:optimization OI}). As mentioned earlier, finding a global optimal solution for this optimization problem is NP-complete \cite{rivest}. Instead, we may take a top down approach and minimize the objective function by minimizing the term $C_a := \pi_{\Theta_a} [1 - H(\rho_a)]$ at each internal node, starting from the root node. Note that the only term that depends on the query chosen at node `$a$' in this cost function is $\rho_a$. Hence the algorithm reduces to minimizing $\rho_a$ (i.e., choosing a split as balanced as possible) at each internal node $a \in \cI$. As a result, $\lambda$-GBS reduces to GBS in this case. Finally, generalized binary search (GBS) is summarized in Algorithm \ref{algo_OIsc1}.

\restylealgo{boxed}
\begin{algorithm}
\dontprintsemicolon
\caption{Greedy decision tree algorithm for minimizing average depth \label{algo_OIsc1}}
\textbf{\underline{Generalized Binary Search (GBS)}} \;
\BlankLine
\textbf{Initialization :}  \emph{Let the leaf set consist of the root node}, $Q_{root} = \emptyset$ \;
\SetLine
\While{some leaf node `$a$' has $|\Theta_a| > 1$}{
\For{each query $q \in Q \setminus Q_a$}{
Find $\Theta_{l(a)}$ and $\Theta_{r(a)}$ produced by making a split with query $q$\;
Compute $\rho_a$ produced by making a split with query $q$\;
}
Choose a query with the least $\rho_a$ at node `$a$'\;
Form child nodes $l(a),r(a)$\;
}
\end{algorithm}


\subsection{Worst case}
\label{sec:OCsc2}

Here, we present the other limiting case of the family of greedy algorithms $\lambda$-GBS, $\lambda \to \infty$. As noted in Section \ref{sec:OI}, the exponential cost function $L_{\lambda}(\cP)$ reduces to the worst case depth of any leaf node in this case. Note that GBS with the uniform prior is an intuitive algorithm for minimizing the worst case depth. Here, we present a theoretical justification for the same. 
\begin{cor}
\label{cor:OIsc2}
In the limiting case when $\lambda \to \infty$, the optimization problem
\begin{align*}
\min \log_{\lambda} \left ( \frac{\pi_{\Theta_{l(a)}}}{\pi_{\Theta_a}} \cD_{\alpha}(\Theta_{l(a)}) + \frac{\pi_{\Theta_{r(a)}}}{\pi_{\Theta_a}} \cD_{\alpha}(\Theta_{r(a)}) \right ) \to \min \max \{ |\Theta_{l(a)}| , |\Theta_{r(a)}| \} 
\end{align*}
\end{cor}
\begin{proof}
Applying L'H\^{o}pital's rule, we get
\begin{align*}
\underset{\lambda \to \infty} {\operatorname{\lim}} \log_{\lambda} \left ( \frac{\pi_{\Theta_{l(a)}}}{\pi_{\Theta_a}} \cD_{\alpha}(\Theta_{l(a)}) + \frac{\pi_{\Theta_{r(a)}}}{\pi_{\Theta_a}} \cD_{\alpha}(\Theta_{r(a)}) \right ) = \max \{ \log_2 |\Theta_{l(a)}|, \log_2 |\Theta_{r(a)}| \}
\end{align*}
Since $\log_2$ is a monotonic increasing function, the optimization problem, $\min \max \{ \log_2 |\Theta_{l(a)}|, \log_2 |\Theta_{r(a)}| \}$ is equivalent to the optimization problem, $\min \max \{ |\Theta_{l(a)}| , |\Theta_{r(a)}| \}$. 
\end{proof}

Note that the cost function minimized at each internal node of a tree in $\lambda$-GBS is $C_a := \frac{\pi_{\Theta_{l(a)}}}{\pi_{\Theta_a}} \cD_{\alpha}(\Theta_{l(a)}) + \frac{\pi_{\Theta_{r(a)}}}{\pi_{\Theta_a}} \cD_{\alpha}(\Theta_{r(a)})$. Since $\log_{\lambda}$ is a monotonic function, this is equivalent to minimizing the function $\log_{\lambda} (C_a)$. We know from Corollary \ref{cor:OIsc2} that in the limiting case when $\lambda$ tends to infinity, this reduces to minimizing $\max\{ |\Theta_{l(a)}|,|\Theta_{r(a)}| \}$. Hence, in this limiting case, $\lambda$-GBS reduces to GBS with uniform prior, thereby completely eliminating the dependence of the algorithm on the prior distribution $\cP$. More generally, as $\lambda$ increases, $\lambda$-GBS becomes less sensitive to the prior distribution, and therefore more robust if the prior is misspecified. 


\section{Group Identification}
\label{sec:GI}

In this section, we consider the problem of identifying the group of an unknown object $\theta \in \Theta$, rather than the object itself, with as few queries as possible. Here, in addition to the binary matrix $\B$ and \emph{a priori} probability distribution $\cP$ on the objects, the group labels for the objects are also provided, where the groups are assumed to be disjoint.

We denote a query learning problem for group identification by $(\B,\cP,\vy)$, where $\vy = (y_1,\cdots,y_M)$ denotes the group labels of the objects, $y_k \in \{1,\cdots,m\}$. Let $\{\Theta^i\}_{i=1}^m$ be the partition of the object set $\Theta$, where $\Theta^i = \{\theta_k \in \Theta: y_k = i\}$. It is important to note here that the group identification problem cannot be simply reduced to a standard query learning problem with groups $\{\Theta^1,\cdots,\Theta^m\}$ as meta ``objects,'' since the objects within a group need not respond the same to each query. For example, consider the toy example shown in Figure \ref{fig:toy network} where the objects $\theta_1,\theta_2$ and $\theta_3$ belonging to group $1$ cannot be considered as one single meta object as these objects respond differently to queries $q_1$ and $q_3$. 

In this context, we also note that GBS can fail to find a good solution for a group identification problem as it does not take the group labels into consideration while choosing queries. Once again, consider the toy example shown in Figure \ref{fig:toy network} where just one query (query $q_2$) is sufficient to identify the group of an unknown object, whereas GBS requires $2$ queries to identify the group when the unknown object is either $\theta_2$ or $\theta_4$. Here, we propose a natural extension of $\lambda$-GBS to the problem of group identification. Specifically, we propose a family of greedy algorithms that aim to minimize the average exponential cost for the problem of group identification.

\begin{figure}[!t]
\begin{minipage}[b]{0.5\linewidth}
\centering
\begin{tabular}{|c|c c c|c|}
\hline
& $q_1$ & $q_2$ & $q_3$ & Group label, $y$ \\
\hline
$\theta_1$ & 0 & 1 & 1 & 1 \\
$\theta_2$ & 1 & 1 & 0 & 1 \\
$\theta_3$ & 0 & 1 & 0 & 1 \\
$\theta_4$ & 1 & 0 & 0 & 2 \\
\hline
\end{tabular}
\caption{Toy Example}
\label{fig:toy network}
\end{minipage}
\hspace{0.35cm}
\begin{minipage}[b]{0.5\linewidth}
\centering
\begin{displaymath}
\xymatrix{
  &  *++[o][F-]{q_1} \ar[dl]_0 \ar[dr]^1 & &  \\
y = 1 & & *++[o][F-]{q_2} \ar[dl]_0 \ar[dr]^1 & \\
& y = 2 & & y = 1 \\
}
\end{displaymath}
\caption{Decision tree constructed using GBS}
\label{fig:toy tree}
\end{minipage}
\end{figure}

Note that when constructing a tree for group identification, a greedy, top-down algorithm, terminates splitting when all the objects at the node belong to the same group. Hence, a tree constructed in this fashion can have multiple objects ending in the same leaf node and multiple leaves ending in the same group. 

For a tree with $L$ leaves, we denote by $\cL^i \subset \cL = \{1,\cdots,L\}$ the set of leaves that terminate in group $i$. Similar to $\Theta^i \subseteq \Theta$, we denote by $\Theta_a^i \subseteq \Theta_a$ the set of objects belonging to group $i$ that reach internal node $a \in \cI$ in the tree.  

Given $(\B,\cP,\vy)$, let $\cT(\B,\cP,\vy)$ denote the set of decision trees that can uniquely identify the groups of all objects in the set $\Theta$. For any decision tree $T \in \cT(\B,\cP,\vy)$, let $d_j$ denote the depth of leaf node $j \in \cL$. Let random variable $X$ denote the exponential cost incurred in identifying the group of an unknown object $\theta \in \Theta$. Then, the average exponential cost $L_{\lambda}(\cP)$ of identifying the group of the unknown object $\theta$ using a given tree is defined as
\begin{align*}
\lambda^{L_{\lambda}(\cP)} & =  \sum_{i=1}^m \prob(\theta \in \Theta^i) \expect[X |\theta \in \Theta^i] \\
 & =  \sum_{i=1}^m \pi_{\Theta^i} \left [ \sum_{j \in \cL^i} \frac{\pi_{\Theta_j}}{\pi_{\Theta^i}} \lambda^{d_j} \right ] \\
\Longrightarrow L_{\lambda}(\cP) & = \log_{\lambda} \left ( \sum_{i=1}^m \pi_{\Theta^i} \left [ \sum_{j \in \cL^i} \frac{\pi_{\Theta_j}}{\pi_{\Theta^i}} \lambda^{d_j} \right ] \right )
\end{align*}
In the limiting case when $\lambda$ tends to one and infinity, the cost function $L_{\lambda}(\cP)$ reduces to
\begin{align*}
L_1(\cP) & := \underset{\lambda \to 1}{\operatorname{\lim}} L_{\lambda}(\cP) = \sum_{i=1}^m \pi_{\Theta^i} \left [ \sum_{j \in \cL^i} \frac{\pi_{\Theta_j}}{\pi_{\Theta^i}} d_j \right ] \\
L_{\infty}(\cP) & := \underset{\lambda \to \infty}{\operatorname{\lim}} L_{\lambda}(\cP) = \underset{j \in \cL}{\operatorname{\max}} \ \ d_j.
\end{align*}

\begin{thm}
\label{thm:GI}
For any $\lambda \geq 1$, the average exponential cost $L_{\lambda}(\cP)$ of identifying the group of an object using a tree $T \in \cT(\B,\cP,\vy)$ is given by
\begin{align}
\label{eq:GI}
L_{\lambda}(\cP) = \log_{\lambda} \left ( \lambda^{H_{\alpha}(\Pi_{\vy})} + \sum_{a \in \cI} \pi_{\Theta_{a}} \left [ (\lambda - 1)\lambda^{d_a} - \cD_{\alpha}(\Theta_a) + \frac{\pi_{\Theta_{l(a)}}}{\pi_{\Theta_a}} \cD_{\alpha}(\Theta_{l(a)}) + \frac{\pi_{\Theta_{r(a)}}}{\pi_{\Theta_a}} \cD_{\alpha}(\Theta_{r(a)})  \right ] \right )
\end{align}
where $\cP_{\vy} = (\pi_{\Theta^1},\cdots,\pi_{\Theta^m})$ denotes the probability distribution of the object groups induced by the labels $\vy$ and $\cD_{\alpha}(\Theta_a) := \left [ \sum_{i=1}^m \left ( \frac{\pi_{\Theta_a^i}}{\pi_{\Theta_a}} \right )^{\alpha} \right ]^{1/\alpha}$ with $\alpha = \frac{1}{1+\log_2\lambda}$, $\pi_{\Theta^i} = \underset{\{k:y_k = i\}}{\operatorname{\sum}} \pi_k$ and $\pi_{\Theta_a^i} = \underset{\{k:\theta_k \in \Theta_a, y_k = i \}}{\operatorname{\sum}} \pi_k$.
\end{thm}
\begin{proof}
See Appendix.
\end{proof}

Note that the definition of $\cD_{\alpha}(\Theta_a)$ in this theorem is a generalization of that in Theorem \ref{thm:OI}. The above theorem states that given a query learning problem for group identification $(\B,\cP,\vy)$, the exponential cost function $L_{\lambda}(\cP)$ is bounded below by the $\alpha$-R\'{e}nyi entropy of the probability distribution of the groups. It also explicitly states the gap in this lower bound. Note that Theorem \ref{thm:OI} is a special case of this theorem where each group is of size $1$.

Using Theorem \ref{thm:GI}, the problem of finding a decision tree with minimum cost function $L_{\lambda}(\cP)$ can be formulated as the following optimization problem:
\begin{eqnarray}
\label{eq:optimization GI}
& \underset{T \in \cT(\B,\cP,\vy)}{\operatorname{\min}}  \sum_{a \in \cI} \pi_{\Theta_{a}} \left [ (\lambda - 1)\lambda^{d_a} - \cD_{\alpha}(\Theta_a) + \frac{\pi_{\Theta_{l(a)}}}{\pi_{\Theta_a}} \cD_{\alpha}(\Theta_{l(a)}) + \frac{\pi_{\Theta_{r(a)}}}{\pi_{\Theta_a}} \cD_{\alpha}(\Theta_{r(a)}) \right ]. &
\end{eqnarray}
This optimization problem being the generalized version of the optimization problem in (\ref{eq:optimization OI}) is NP-complete. Hence, we propose a suboptimal approach to solve this optimization problem where we solve the objective function locally instead of globally. We take a top-down approach and minimize the objective function by minimizing the term $C_a := \frac{\pi_{\Theta_{l(a)}}}{\pi_{\Theta_a}} \cD_{\alpha}(\Theta_{l(a)}) + \frac{\pi_{\Theta_{r(a)}}}{\pi_{\Theta_a}} \cD_{\alpha}(\Theta_{r(a)})$ at each internal node, starting from the root node. The algorithm, which we refer to as $\lambda$-GGBS, is summarized in Algorithm \ref{algo_GI}.

\restylealgo{boxed}
\begin{algorithm}
\dontprintsemicolon
\caption{Greedy decision tree algorithm for group identification that minimizes average exponential cost \label{algo_GI}}

\textbf{\underline{$\lambda$ Group identification Generalized Binary Search ($\lambda$-GGBS)}} \;
\BlankLine
\textbf{Initialization :}  \emph{Let the leaf set consist of the root node}, $Q_{root} = \emptyset$ \;
\SetLine
\While{some leaf node `$a$' has more than one group of objects}{
\For{each query $q \in Q \setminus Q_a$}{
Compute $\{\Theta_{l(a)}^i\}_{i=1}^m$ and $\{\Theta_{r(a)}^i\}_{i=1}^m$ produced by making a split with query $q$\;
Compute the cost $C_a(q)$ of making a split with query $q$\;
}
Choose a query with the least cost $C_a$ at node `$a$'\;
Form child nodes $l(a),r(a)$\;
}
\end{algorithm}


\subsection{Average case}
\label{sec:GIsc1}

The interpretation of $\lambda$-GGBS is somewhat easier in the limiting case when $\lambda$ tends to one. In addition to the reduction factor defined in Section \ref{sec:OIsc1}, we define a new parameter called the \emph{group reduction factor} for each group $i \in \{1,\cdots,m\}$ at each internal node.

\begin{defn}[\textbf{Group reduction factor}] Let $T$ be a decision tree constructed on the query learning problem for group identification $(\B,\cP,\vy)$. The \emph{group reduction factor} for any group $i$ at an internal node `$a$' in the tree is defined by $\rho_a^i = \max\{\pi_{\Theta_{l(a)}^i},\pi_{\Theta_{r(a)}^i}\}/\pi_{\Theta_a^i}$.
\end{defn} 

\begin{cor}
\label{cor:GIsc1}
The expected number of queries required to identify the group of an unknown object using a tree $T \in \cT(\B,\cP,\vy)$ is given by
\begin{equation}
\label{eq:GIsc1}
L_1(\cP) = H(\cP_{\vy}) + \sum_{a \in \cI} \pi_{\Theta_{a}}\left [1 - H(\rho_a) + \sum_{i=1}^m \frac{\pi_{\Theta_a^i}}{\pi_{\Theta_a}} H(\rho_a^i)  \right ] 
\end{equation}
where $\cP_{\vy} = (\pi_{\Theta^1},\cdots,\pi_{\Theta^m})$ denotes the probability distribution of the object groups induced by the labels $\vy$ and $H(\cdot)$ denotes the Shannon entropy.
\end{cor}
\begin{proof}
The result follows from Theorem \ref{thm:GI} by taking the limit as $\lambda$ tends to $1$ and applying L'H\^{o}pital's rule on both sides of (\ref{eq:GI}).
\end{proof}

This corollary states that given a query learning problem for group identification $(\B,\cP,\vy)$, the expected number of queries required to identify the group of an unknown object is lower bounded by the Shannon entropy of the probability distribution of the groups. It also follows from the above result that this lower bound is achieved iff the reduction factor $\rho_a$ is equal to $0.5$ and the group reduction factors $\{\rho_a^i\}_{i=1}^m$ are equal to $1$ at every internal node in the tree. Also, note that the result in Corollary \ref{cor:OIsc1} is a special case of this result where each group is of size $1$ leading to $\rho_a^i = 1$ for all groups at every internal node.

Using this result, the problem of finding a decision tree with minimum $L_1(\cP)$ can be formulated as the following optimization problem:
\begin{eqnarray}
\label{eq:optimization GIsc1}
& \underset{T \in \cT(\B,\cP,\vy)}{\operatorname{\min}}  \sum_{a \in \cI} \pi_{\Theta_a}\left [1 - H(\rho_a) + \sum_{i=1}^m \frac{\pi_{\Theta_a^i}}{\pi_{\Theta_a}} H(\rho_a^i) \right ].  &
\end{eqnarray} 
We propose a greedy top-down approach and minimize the objective function by minimizing the term $ \pi_{\Theta_a} [1  $ $ - H(\rho_a) + \sum_{i=1}^m \frac{\pi_{\Theta_a^i}}{\pi_{\Theta_a}} H(\rho_a^i) ]$ at each internal node, starting from the root node. Note that the terms that depend on the query chosen at node `$a$' are $\rho_a$ and $\rho_a^i$. Hence the algorithm reduces to minimizing $C_a := 1 - H(\rho_a) + \sum_{i=1}^m \frac{\pi_{\Theta_a^i}}{\pi_{\Theta_a}} H(\rho_a^i)$ at each internal node $a \in \cI$. Note that this objective function consists of two terms, the first term $[1 - H(\rho_a)]$ favors queries that evenly distribute the probability mass of the objects at node `$a$' to its child nodes (regardless of the group) while the second term $\sum_i \frac{\pi_{\Theta_a^i}}{\pi_{\Theta_a}} H(\rho_a^i)$ favors queries that transfer an entire group of objects to one of its child nodes. The algorithm, which we refer to as GGBS, is summarized in Algorithm \ref{algo_GIsc1}.

\restylealgo{boxed}
\begin{algorithm}
\dontprintsemicolon
\caption{Greedy decision tree algorithm for group identification that minimizes average linear cost\label{algo_GIsc1}}

\textbf{\underline{Group identification Generalized Binary Search (GGBS)}} \;
\BlankLine
\textbf{Initialization :}  \emph{Let the leaf set consist of the root node}, $Q_{root} = \emptyset$ \;
\SetLine
\While{some leaf node `$a$' has more than one group of objects}{
\For{each query $q \in Q \setminus Q_a$}{
Compute $\{\rho_{a}^i\}_{i=1}^m$ and $\rho_a$ produced by making a split with query $q$\;
Compute the cost $C_a(q)$ of making a split with query $q$\;
}
Choose a query with the least cost $C_a$ at node `$a$'\;
Form child nodes $l(a),r(a)$\;
}
\end{algorithm}

There is an interesting connection between the above algorithm and impurity-based decision tree induction. In particular, the above algorithm is equivalent to the decision tree splitting algorithm used in C$4.5$ software package \cite{quinlan}, based on the entropy impurity measure. See \cite{bellala} for more details on this relation. 


\subsection{Worst case}
\label{sec:GIsc2}
We now present $\lambda$-GGBS in the limiting case when $\lambda$ tends to infinity. As noted in Section \ref{sec:GI}, the exponential cost function $L_{\lambda}(\cP)$ reduces to the worst case depth of any leaf node in this limiting case. Let $N_a$ denote the number of groups at any node `$a$' in the tree, i.e., $N_a = |\{i \in \{1,\cdots,m\}:\Theta_a^i \neq \emptyset\}|$.

\begin{cor}
\label{cor:GIsc2}
In the limiting case when $\lambda \to \infty$, the optimization problem
\begin{align*}
\min \log_{\lambda} \left ( \frac{\pi_{\Theta_{l(a)}}}{\pi_{\Theta_a}} \cD_{\alpha}(\Theta_{l(a)}) + \frac{\pi_{\Theta_{r(a)}}}{\pi_{\Theta_a}} \cD_{\alpha}(\Theta_{r(a)}) \right ) \longrightarrow \min \max \{ N_{l(a)} , N_{r(a)}\} 
\end{align*}
where $\cD_{\alpha}(\Theta_a) = \left [ \sum_{i=1}^m \left ( \frac{\pi_{\Theta_a^i}}{\pi_{\Theta_a}} \right )^{\alpha} \right ]^{\frac{1}{\alpha}}$
\end{cor}
\begin{proof}
Applying L'H\^{o}pital's rule, we get
\begin{align*}
\underset{\lambda \to \infty} {\operatorname{\lim}} \log_{\lambda} \left ( \frac{\pi_{\Theta_{l(a)}}}{\pi_{\Theta_a}} \cD_{\alpha}(\Theta_{l(a)}) + \frac{\pi_{\Theta_{r(a)}}}{\pi_{\Theta_a}} \cD_{\alpha}(\Theta_{r(a)}) \right ) = \max \{ \log_2 N_{l(a)}, \log_2 N_{r(a)} \}
\end{align*}
Since $\log_2$ is a monotonic increasing function, the optimization problem, $\min \max \{ \log_2 N_{l(a)}, \log_2 N_{r(a)} \}$ is equivalent to the optimization problem, $\min \max \{ N_{l(a)} , N_{r(a)} \}$. 
\end{proof}

Note that the cost function minimized at each internal node of a tree in $\lambda$-GGBS is $C_a := \frac{\pi_{\Theta_{l(a)}}}{\pi_{\Theta_a}} \cD_{\alpha}(\Theta_{l(a)}) + \frac{\pi_{\Theta_{r(a)}}}{\pi_{\Theta_a}} \cD_{\alpha}(\Theta_{r(a)})$. Since $\log_{\lambda}$ is a monotonic function, this is equivalent to minimizing the function $\log_{\lambda} (C_a)$. We know from Corollary \ref{cor:GIsc2} that in the limiting case when $\lambda$ tends to infinity, this reduces to minimizing $\max \{ N_{l(a)} , N_{r(a)} \}$ at each internal node in the tree.


\begin{figure}[!t]
  \centering
  \includegraphics[scale = 0.42]{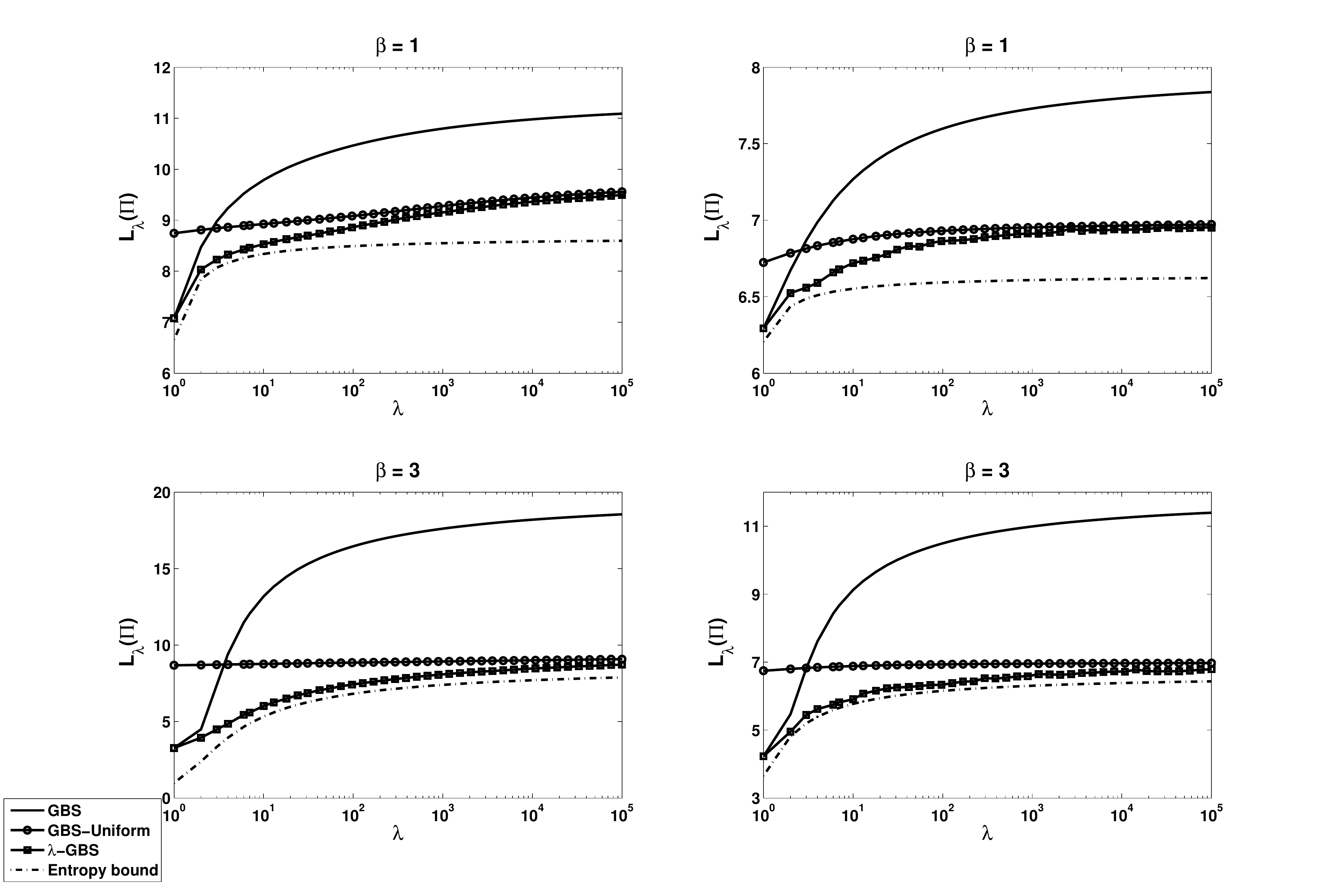}
  \caption{\small \sl Experiments to demonstrate the improved performance of $\lambda$-GBS over GBS and GBS with uniform prior. The plots in the first column correspond to the WISER database and those in the second column correspond to synthetic data.}
 \label{fig:exp cost}
\end{figure}

\section{Experiments}
We compare the proposed algorithms with GBS on both synthetic data and a real dataset known as WISER, which is a toxic chemical database describing the binary relationship between $298$ toxic chemicals and $79$ acute symptoms. We only present results for object (as opposed to group) identification. Figure \ref{fig:exp cost} demonstrates the improved performance of $\lambda$-GBS over standard GBS, and GBS with uniform prior, over a range of $\lambda$ values. Each curve corresponds to the average value of the cost function $L_{\lambda}(\cP)$ as a function of $\lambda$ over $100$ repetitions. 

The plots in the first column correspond to the WISER database, which has been studied in more detail in \cite{bellala}.
Here, in each repetition, the prior is generated according to Zipf's law, i.e., $(k^{-\beta}/\sum_{i=1}^M i^{-\beta})_{k=1}^M$, $\beta \geq 0$, after randomly permuting the objects. Note that in the special case, when $\beta = 0$, this reduces to the uniform distribution and as $\beta$ increases, it tends to a skewed distribution with most of the probability mass concentrated on a single object. 

The plots in the second column correspond to synthetic data based on an active learning application. We consider a two-dimensional setting where the classifiers are restricted to be linear classifiers of the form $sign(x_i - c)$, $sign(c - x_i)$, where $i=1,2$ and $c$ takes on $25$ distinct values. The number of distinct classifiers is therefore $100$, and the number of queries is $26^2 = 676$. The goal is to identify the classifier by selecting queries judiciously. Here, the prior is generated such that the classifiers that are close to $x_i = 0$ are more likely than the ones away from the axes, with their relative probability decreasing according to Zipf's law $k^{-\beta}$, $\beta \geq 0$. Hence, the prior is the same in each repetition. However, the randomness in each repetition comes from the greedy algorithms due to the presence of multiple best splits at each internal node. Note that in all the experiments, $\lambda$-GBS performs better than GBS and GBS with uniform prior. We also see that $\lambda$-GBS converges to GBS as $\lambda \to 1$ and to GBS with uniform prior as $\lambda \to \infty$.


\section{Conclusions}
In this paper, we show that generalized binary search (GBS) is a greedy algorithm to optimize the expected number of queries needed to identify an object. We develop two extensions of GBS, motivated by the problem of toxic chemical identification. First, we derive a greedy algorithm, $\lambda$-GBS, to minimize the expected exponentially weighted query cost. The average and worst cases fall out in the limits as $\lambda \to 1$ and $\lambda \to \infty$, and correspond to GBS and GBS with uniform prior, respectively. Second, we suppose the objects are partitioned into groups, and the goal is to identify only the group of the unknown object. Once again, we propose a greedy algorithm, $\lambda$-GGBS, to minimize the expected exponentially weighted query cost. The algorithms are derived in a common framework. In particular, we prove exact formulas for the exponentially weighted query cost that close the gap between previously known lower bounds related to R\'{e}nyi entropy. These exact formulas are then optimized in a greedy, top-down manner to construct a decision tree. An interesting open question is to relate these greedy algorithms to the global optimizer of the exponentially weighted cost function. \\ 

\noindent 
\textbf{Acknowledgments:}
G. Bellala and C. Scott were supported in part by NSF Awards No. 0830490 and 0953135. S. Bhavnani was supported in part by NIH grant No. UL1RR024986. The authors would like to thank B. Mashayekhi and P. Wexler from NLM for providing access to the WISER database.

\bibliographystyle{IEEEtran}
\bibliography{ref}

\begin{thebibliography}{10}
\providecommand{\url}[1]{#1}
\csname url@samestyle\endcsname
\providecommand{\newblock}{\relax}
\providecommand{\bibinfo}[2]{#2}
\providecommand{\BIBentrySTDinterwordspacing}{\spaceskip=0pt\relax}
\providecommand{\BIBentryALTinterwordstretchfactor}{4}
\providecommand{\BIBentryALTinterwordspacing}{\spaceskip=\fontdimen2\font plus
\BIBentryALTinterwordstretchfactor\fontdimen3\font minus
  \fontdimen4\font\relax}
\providecommand{\BIBforeignlanguage}[2]{{%
\expandafter\ifx\csname l@#1\endcsname\relax
\typeout{** WARNING: IEEEtran.bst: No hyphenation pattern has been}%
\typeout{** loaded for the language `#1'. Using the pattern for}%
\typeout{** the default language instead.}%
\else
\language=\csname l@#1\endcsname
\fi
#2}}
\providecommand{\BIBdecl}{\relax}
\BIBdecl

\bibitem{koren}
I.~Koren and Z.~Kohavi, ``Diagnosis of intermittent faults in combinational
  networks,'' \emph{IEEE Transactions on Computers}, vol. C-26, pp. 1154--1158,
  1977.

\bibitem{unluyurt}
\"{U}nl\"{u}yurt, ``Sequential testing of complex systems: A review,''
  \emph{Discrete Applied Mathematics}, vol. 142, no. 1-3, pp. 189--205, 2004.

\bibitem{shiozaki}
J.~Shiozaki, H.~Matsuyama, E.~O'Shima, and M.~Iri, ``An improved algorithm for
  diagnosis of system failures in the chemical process,'' \emph{Computational
  Chemical Engineering}, vol.~9, no.~3, pp. 285--293, 1985.

\bibitem{loveland}
D.~W. Loveland, ``Performance bounds for binary testing with arbitrary
  weights,'' \emph{Acta Informatica}, 1985.

\bibitem{pattipati}
F.~Yu, F.~Tu, H.~Tu, and K.~Pattipati, ``Multiple disease (fault) diagnosis
  with applications to the QMR-DT problem,'' \emph{Proceedings of IEEE
  International Conference on Systems, Man and Cybernetics}, vol.~2, pp.
  1187--1192, October 2003.

\bibitem{geman}
D.~Geman and B.~Jedynak, ``An active testing model for tracking roads in
  satellite images,'' \emph{IEEE Transactions on Pattern Analysis and Machine
  Intelligence}, vol.~18, no.~1, pp. 1--14, 1996.

\bibitem{dasgupta}
S.~Dasgupta, ``Analysis of a greedy active learning strategy,'' \emph{Advances
  in Neural Information Processing Systems}, 2004.

\bibitem{nowak}
R.~Nowak, ``Generalized binary search,'' \emph{Proceedings of the Allerton
  Conference}, 2008.

\bibitem{angluin}
D.~Angluin, ``Queries revisited,'' \emph{Theoretical Computer Science}, vol.
  313, pp. 175--194, 2004.

\bibitem{szczur}
M.~Szczur and B.~Mashayekhi, ``WISER Wireless information system for emergency
  responders,'' \emph{Proceedings of American Medical Informatics Association
  Annual Symposium}, 2005.

\bibitem{goodman}
R.~M. Goodman and P.~Smyth, ``Decision tree design from a communication theory
  standpoint,'' \emph{IEEE Transactions on Information Theory}, vol.~34, no.~5,
  1988.

\bibitem{shannon}
C.~E. Shannon, ``A mathematical theory of communication,'' \emph{Bell Systems
  Technical Journal}, vol.~27, pp. 379 -- 423, July 1948.

\bibitem{fano}
R.~M. Fano, \emph{Transmission of Information}.\hskip 1em plus 0.5em minus
  0.4em\relax MIT Press, 1961.

\bibitem{huffman}
D.~A. Huffman, ``A method for the construction of minimum-redundancy codes,''
  \emph{Proceedings of the Institute of Radio Engineers}, 1952.

\bibitem{cover}
T.~M. Cover and J.~A. Thomas, \emph{Elements of Information Theory}.\hskip 1em
  plus 0.5em minus 0.4em\relax John Wiley, 1991.

\bibitem{garey1}
M.~Garey, ``Optimal binary decision trees for diagnostic identification
  problems,'' Ph.D. dissertation, University of Wisconsin, Madison, 1970.

\bibitem{garey2}
------, ``Optimal binary identification procedures,'' \emph{SIAM Journal on
  Applied Mathematics}, vol. 23(2), pp. 173--186, 1972.

\bibitem{rivest}
L.~Hyafil and R.~Rivest, ``Constructing optimal binary decision trees is
  NP-complete,'' \emph{Information Processing Letters}, vol. 5(1), pp. 15--17,
  1976.

\bibitem{kosaraju}
S.~R. Kosaraju, T.~M. Przytycka, and R.~S. Borgstrom, ``On an optimal split
  tree problem,'' \emph{Proceedings of 6th International Workshop on Algorithms
  and Data Structures, WADS}, pp. 11--14, 1999.

\bibitem{roy}
S.~Roy, H.~Wang, G.~Das, U.~Nambiar, and M.~Mohania, ``Minimum-effort driven
  dynamic faceted search in structured databases,'' \emph{Proceedings of the
  17th ACM Conference on Information and Knowledge Management}, pp. 13--22,
  2008.

\bibitem{campbell1}
L.~L. Campbell, ``A coding problem and R\'{e}nyi's entropy,'' \emph{Information
  and Control}, vol.~8, no.~4, pp. 423--429, August 1965.

\bibitem{baer1}
M.~B. Baer, ``R\'{e}nyi to R\'{e}nyi - source coding under seige,''
  \emph{Proceedings of IEEE International Symposium on Information Theory}, pp.
  1258--1262, July 2006.

\bibitem{schulz}
F.~Schulz, ``Trees with exponentially growing costs,'' \emph{Information and
  Computation}, vol. 206, 2008.

\bibitem{campbell2}
L.~L. Campbell, ``Definition of entropy by means of a coding problem,''
  \emph{Z.Wahrscheinlichkeitstheorie und verwandte Gebiete}, vol.~6, pp.
  113--118, 1966.

\bibitem{hu}
T.~C. Hu, D.~J. Kleitman, and J.~T. Tamaki, ``Binary trees optimal under
  various criteria,'' \emph{SIAM Journal on Applied Mathematics}, vol.~37,
  no.~2, pp. 246--256, October 1979.

\bibitem{parker}
D.~S. Parker, ``Conditions for the optimality of the Huffman algorithm,''
  \emph{SIAM Journal on Computing}, vol.~9, no.~3, pp. 470--489, August 1980.

\bibitem{humblet}
P.~A. Humblet, ``Generalization of Huffman coding to minimize the probability
  of buffer overflow,'' \emph{IEEE Transactions on Information Theory}, vol.
  IT-27, no.~2, pp. 230--232, March 1981.

\bibitem{quinlan}
J.~R. Quinlan, \emph{C4.5: Programs for Machine Learning}.\hskip 1em plus 0.5em
  minus 0.4em\relax Morgan Kaufmann Publishers, 1993.

\bibitem{bellala}
G.~Bellala, S.~Bhavnani, and C.~Scott, ``Group-based query learning for rapid
  diagnosis in time-critical situations,'' University of Michigan, Ann Arbor,
  Tech. Rep., 2009.

\end{thebibliography}


\appendix

\section{Proof of Theorem~\ref{thm:GI}}
Define two new functions $\Lt_{\lambda}$ and $\Ht_{\alpha}$ as
\begin{align*}
\Lt_{\lambda} & := \frac{1}{\lambda - 1} \left [ \sum_{j \in \cL} \pi_{\Theta_j} \lambda^{d_j} - 1 \right ] = \sum_{j \in \cL} \pi_{\Theta_j} \left [ \sum_{k=0}^{d_j-1} \lambda^k \right ] \\
\Ht_{\alpha} & := 1- \frac{1}{ \left (\sum_{i=1}^m \pi_{\Theta^i}^{\alpha} \right )^{\frac{1}{\alpha}}}
\end{align*}
Noting that the cost function $L_{\lambda}(\cP)$ can be written as,
\begin{align*}
L_{\lambda}(\cP) & = \log_{\lambda} \left ( \sum_{j \in \cL} \pi_{\Theta_j} \lambda^{d_j} \right ),
\end{align*}
the new function $\Lt_{\lambda}$ can be related to the cost function $L_{\lambda}(\Pi)$ as
\begin{align}
\label{eq:relation L Lt}
\lambda^{L_{\lambda}(\Pi)} = (\lambda - 1) \Lt_{\lambda} + 1 
\end{align}
Similarly, $\Ht_{\alpha}$ is related to the $\alpha$-R\'{e}nyi entropy $H_{\alpha}(\cP_{\vy})$ as
\begin{subequations}
\begin{align}
H_{\alpha}(\cP_{\vy}) & = \frac{1}{1 - \alpha} \log_2 \sum_{i=1}^m \pi_{\Theta^i}^{\alpha} = \frac{1}{\alpha \log_2 \lambda} \log_2 \sum_{i=1}^m \pi_{\Theta^i}^{\alpha} = \log_{\lambda} \left ( \sum_{i=1}^m \pi_{\Theta^i}^{\alpha} \right )^{\frac{1}{\alpha}} \label{relation_H_Ht_first} \\
\Longrightarrow \lambda^{H_{\alpha}(\cP_{\vy})} & = \left ( \sum_{i=1}^m \pi_{\Theta^i}^{\alpha} \right )^{\frac{1}{\alpha}} = \left ( \sum_{i=1}^m \pi_{\Theta^i}^{\alpha} \right )^{\frac{1}{\alpha}} \Ht_{\alpha} + 1 \label{relation_H_Ht_second}
\end{align}
\end{subequations}
where we use the definition of $\alpha$, i.e., $\alpha = \frac{1}{1 + \log_2 \lambda}$ in (\ref{relation_H_Ht_first}).

Now, we note from Lemma \ref{lem:Llambda decomposition} that $\Lt_{\lambda}$ can be decomposed as
\begin{align}
\label{eq:Llambda decomposition}
\Lt_{\lambda} & = \sum_{a \in \cI} \lambda^{d_a} \pi_{\Theta_a} \nonumber \\
\Longrightarrow \lambda^{L_{\lambda}(\Pi)} & = 1 + \sum_{a \in \cI} (\lambda - 1) \lambda^{d_a} \pi_{\Theta_a}
\end{align}
where $d_a$ denotes the depth of internal node `$a$' in the tree $T$. Similarly, note from Lemma \ref{lem:Halpha decomposition} that $\Ht_{\alpha}$ can be decomposed as
\begin{align}
\label{eq:Halpha decomposition}
\Ht_{\alpha} & =  \frac{1}{\left ( \sum_{i=1}^m \pi_{\Theta^i}^{\alpha} \right )^{\frac{1}{\alpha}}} \sum_{a \in \cI} \left [ \pi_{\Theta_a} \cD_{\alpha}(\Theta_a) -  \pi_{\Theta_{l(a)}} \cD_{\alpha}(\Theta_{l(a)}) -  \pi_{\Theta_{r(a)}} \cD_{\alpha}(\Theta_{r(a)}) \right ] \nonumber \\
\Longrightarrow \lambda^{H_{\alpha}(\cP_{\vy})} & = 1 + \sum_{a \in \cI} \left [ \pi_{\Theta_a} \cD_{\alpha}(\Theta_a) -  \pi_{\Theta_{l(a)}} \cD_{\alpha}(\Theta_{l(a)}) - \pi_{\Theta_{r(a)}} \cD_{\alpha}(\Theta_{r(a)}) \right ].
\end{align}
Finally, the result follows from (\ref{eq:Llambda decomposition}) and (\ref{eq:Halpha decomposition}) above. 
 
 
\begin{lemma}
\label{lem:Llambda decomposition}
The function $\Lt_{\lambda}$ can be decomposed over the internal nodes in a tree $T$, as
\begin{align*}
\Lt_{\lambda} & = \sum_{a \in \cI} \lambda^{d_a} \pi_{\Theta_a} 
\end{align*}
where $d_a$ denotes the depth of internal node $a \in \cI$ and $\pi_{\Theta_a}$ is the probability mass of the objects at that node.
\end{lemma}
\begin{proof}
Let $T_a$ denote a subtree from any internal node `$a$' in the tree $T$ and let $\cI_a, \cL_a$ denote the set of internal nodes and leaf nodes in the subtree $T_a$, respectively. Then, define $\Lt_{\lambda}^a$ in the subtree $T_a$ to be
\begin{eqnarray*}
& \Lt_{\lambda}^a = \sum_{j \in \cL_a} \frac{\pi_{\Theta_j}}{\pi_{\Theta_a}} \left [ \sum_{k=0}^{d_j^a-1} \lambda^k \right ] & \\
\end{eqnarray*}
where $d_j^a$ denotes the depth of leaf node $j \in \cL_a$ in the subtree $T_a$. 

Now, we show using induction that for any subtree $T_a$ in the tree $T$, the following relation holds
\begin{align}
\label{eq:Lt decomposition}
\pi_{\Theta_a}\Lt_{\lambda}^a & = \sum_{s \in \cI_a} \lambda^{d_s^a} \pi_{\Theta_s}
\end{align}
where $d_s^a$ denotes the depth of internal node $s \in \cI_a$ in the subtree $T_a$.

The relation holds trivially for any subtree $T_a$ rooted at an internal node $a \in \cI$ whose both child nodes terminate as leaf nodes, with both the left hand side and the right hand side of the expression equal to $\pi_{\Theta_a}$. Now, consider a subtree $T_a$ rooted at an internal node $a \in \cI$ whose left child (or right child) alone terminates as a leaf node. Assume that the above relation holds true for the subtree rooted at the right child of node `$a$'. Then,
\begin{align*}
\pi_{\Theta_a} \Lt_{\lambda}^a &= \sum_{j \in \cL_a} \pi_{\Theta_j} \left [ \sum_{k=0}^{d_j^a-1} \lambda^k \right ] \\
& = \sum_{\{j \in \cL_a : d_j^a = 1 \}} \pi_{\Theta_j} + \sum_{\{j \in \cL_a: d_j^a > 1 \}} \pi_{\Theta_j} \left [ \sum_{k=0}^{d_j^a-1} \lambda^k \right ] \\
& = \pi_{\Theta_{l(a)}} + \sum_{\{j \in \cL_a: d_j^a > 1 \}} \pi_{\Theta_j} \left [ 1+ \lambda \sum_{k=0}^{d_j^a-2} \lambda^k \right ] \\
& = \pi_{\Theta_a} + \lambda \sum_{j \in \cL_{r(a)}} \pi_{\Theta_j} \left [ \sum_{k=0}^{d_j^{r(a)}-1} \lambda^k \right ] \\
 & = \pi_{\Theta_a} + \lambda \sum_{s \in \cI_{r(a)}} \lambda^{d_s^{r(a)}} \pi_{\Theta_s}
\end{align*}
where the last step follows from the induction hypothesis. Finally, consider a subtree $T_a$ rooted at an internal node $a \in \cI$ whose neither child node terminates as a leaf node. Assume that the relation in (\ref{eq:Lt decomposition}) holds true for the subtrees rooted at its left and right child nodes. Then,
\begin{align*}
\pi_{\Theta_a} \Lt_{\lambda}^a & = \sum_{j \in \cL_a} \pi_{\Theta_j} \left [ \sum_{k=0}^{d_j^a-1} \lambda^k \right ]  \\
& = \sum_{j \in \cL_{l(a)}} \pi_{\Theta_j} \left [ 1 + \lambda \sum_{k=0}^{d_j^a-2} \lambda^k \right ] + \sum_{j \in \cL_{r(a)}} \pi_{\Theta_j} \left [ 1 + \lambda \sum_{k=0}^{d_j^a-2} \lambda^k \right ] \\
& = \pi_{\Theta_a} + \lambda \sum_{j \in \cL_{l(a)}} \pi_{\Theta_j} \left [ \sum_{k=0}^{d_j^{l(a)}-1} \lambda^k \right ]  + \lambda \sum_{j \in \cL_{r(a)}} \pi_{\Theta_j} \left [ \sum_{k=0}^{d_j^{r(a)}-1} \lambda^k \right ] \\
& = \pi_{\Theta_a} + \lambda \left [ \sum_{s \in \cI_{l(a)}} \lambda^{d_s^{l(a)}} \pi_{\Theta_s} + \sum_{s \in \cI_{r(a)}} \lambda^{d_s^{r(a)}} \pi_{\Theta_s} \right ] = \sum_{s \in \cI_a} \lambda^{d_s^a} \pi_{\Theta_s}
\end{align*}
thereby completing the induction. Finally, the result follows by applying the relation in (\ref{eq:Lt decomposition}) to the tree $T$ whose probability mass at the root node, $\pi_{\Theta_a} = 1$.
\end{proof}


\begin{lemma}
\label{lem:Halpha decomposition}
The function $\Ht_{\alpha}$ can be decomposed over the internal nodes in a tree $T$, as
\begin{align*}
\Ht_{\alpha} & = \frac{1}{\left ( \sum_{i=1}^m \pi_{\Theta^i}^{\alpha} \right )^{\frac{1}{\alpha}}} \sum_{a \in \cI} \left [ \pi_{\Theta_a} \cD_{\alpha}(\Theta_a) -  \pi_{\Theta_{l(a)}} \cD_{\alpha}(\Theta_{l(a)}) -  \pi_{\Theta_{r(a)}} \cD_{\alpha}(\Theta_{r(a)}) \right ] 
\end{align*}
where $\cD_{\alpha}(\Theta_a) := \left [ \sum_{i=1}^m \left (\frac{\pi_{\Theta_a^i}}{\pi_{\Theta_a}} \right )^{\alpha} \right ]^{\frac{1}{\alpha}}$ and $\pi_{\Theta_a}$ denotes the probability mass of the objects at any internal node $a \in \cI$.
\end{lemma}
\begin{proof}
Let $T_a$ denote a subtree from any internal node `$a$' in the tree $T$ and let $\cI_a$ denote the set of internal nodes in the subtree $T_a$. Then, define $\Ht_{\alpha}^a$ in a subtree $T_a$ to be
\begin{eqnarray*}
& \Ht_{\alpha}^a = 1 - \frac{\pi_{\Theta_a}}{\left [ \sum_{i=1}^m \pi_{\Theta_a^i}^{\alpha} \right ]^{\frac{1}{\alpha}}}  &
\end{eqnarray*}

Now, we show using induction that for any subtree $T_a$ in the tree $T$, the following relation holds
\begin{align}
\label{eq:Ht decomposition}
\left [ \sum_{i=1}^m \pi_{\Theta_a^i}^{\alpha} \right ]^{\frac{1}{\alpha}} \Ht_{\alpha}^a & = \sum_{s \in \cI_a} \left [ \pi_{\Theta_s} \cD_{\alpha}(\Theta_s) -  \pi_{\Theta_{l(s)}} \cD_{\alpha}(\Theta_{l(s)}) -  \pi_{\Theta_{r(s)}} \cD_{\alpha}(\Theta_{r(s)}) \right ]
\end{align}

Note that the relation holds trivially for any subtree $T_a$ rooted at an internal node $a \in \cI$ whose both child nodes terminate as leaf nodes. Now, consider a subtree $T_a$ rooted at any other internal node $a \in \cI$. Assume the above relation holds true for the subtrees rooted at its left and right child nodes. Then,
\begin{align*}
\left [ \sum_{i=1}^m \pi_{\Theta_a^i}^{\alpha} \right ]^{\frac{1}{\alpha}} \Ht_{\alpha}^a & = \left [ \sum_{i=1}^m \pi_{\Theta_a^i}^{\alpha} \right ]^{\frac{1}{\alpha}} - \pi_{\Theta_a} \\
& = \left [ \sum_{i=1}^m \pi_{\Theta_a^i}^{\alpha} \right ]^{\frac{1}{\alpha}} - \pi_{\Theta_{l(a)}} - \pi_{\Theta_{r(a)}} \\
& = \left [ \sum_{i=1}^m \pi_{\Theta_a^i}^{\alpha} \right ]^{\frac{1}{\alpha}} - \left [ \sum_{i=1}^m \pi_{\Theta_{l(a)}^i}^{\alpha} \right ]^{\frac{1}{\alpha}} - \left [ \sum_{i=1}^m \pi_{\Theta_{r(a)}^i}^{\alpha} \right ]^{\frac{1}{\alpha}} \\
& \ \ \ \ + \left ( \left [ \sum_{i=1}^m \pi_{\Theta_{l(a)}^i}^{\alpha} \right ]^{\frac{1}{\alpha}} - \pi_{\Theta_{l(a)}} \right ) + \left ( \left [ \sum_{i=1}^m \pi_{\Theta_{r(a)}^i}^{\alpha} \right ]^{\frac{1}{\alpha}} - \pi_{\Theta_{r(a)}} \right ) \\
& =  \left [ \pi_{\Theta_a} \cD_{\alpha}(\Theta_a) -  \pi_{\Theta_{l(a)}} \cD_{\alpha}(\Theta_{l(a)}) -  \pi_{\Theta_{r(a)}} \cD_{\alpha}(\Theta_{r(a)}) \right ] \\
& \ \ \ \ + \left [ \sum_{i=1}^m \pi_{\Theta_{l(a)}^i}^{\alpha} \right ]^{\frac{1}{\alpha}} \Ht_{\alpha}^{l(a)} + \left [ \sum_{i=1}^m \pi_{\Theta_{r(a)}^i}^{\alpha} \right ]^{\frac{1}{\alpha}} \Ht_{\alpha}^{r(a)} \\
& = \sum_{s \in \cI_a} \left [ \pi_{\Theta_s} \cD_{\alpha}(\Theta_s) -  \pi_{\Theta_{l(s)}} \cD_{\alpha}(\Theta_{l(s)}) -  \pi_{\Theta_{r(s)}} \cD_{\alpha}(\Theta_{r(s)}) \right ]
\end{align*}
where the last step follows from the induction hypothesis. Finally, the result follows by applying the relation in (\ref{eq:Ht decomposition}) to the tree $T$. 
\end{proof}

\end{document}